\documentclass[iicol,sn-mathphys-num]{sn-jnl}

\usepackage{graphicx}%
\usepackage{multirow}%
\usepackage{amsmath,amssymb,amsfonts}%
\usepackage{amsthm}%
\usepackage{mathrsfs}%
\usepackage[title]{appendix}%
\usepackage{xcolor}%
\usepackage{textcomp}%
\usepackage{manyfoot}%
\usepackage{booktabs}%
\usepackage{algorithm}%
\usepackage{algorithmicx}%
\usepackage{algpseudocode}%
\usepackage{listings}%
\usepackage{bm}

\theoremstyle{thmstyleone}
\newtheorem{theorem}{Theorem}
\theoremstyle{thmstyletwo}%
\newtheorem{remark}{Remark}%

\theoremstyle{thmstylethree}%

\begin{document}

\title[Article Title]{Avian-Inspired High-Precision Tracking Control for Aerial Manipulators}


\author[1,2]{\fnm{Mengyu} \sur{Ji}}\email{jimengyu@westlake.edu.cn}

\author[2]{\fnm{Jiahao} \sur{Shen}}\email{shenjiahao@westlake.edu.cn}

\author*[2,3]{\fnm{Huazi} \sur{Cao}}\email{caohuazi@wioe.westlake.edu.cn}

\author[2,4,5]{\fnm{Shiyu} \sur{Zhao}}\email{zhaoshiyu@westlake.edu.cn}

\affil[1]{\orgdiv{College of Computer Science and Technology}, \orgname{ Zhejiang University}, \orgaddress{\city{Hangzhou 310058}, \state{Zhejiang}, \country{China}}}

\affil*[2]{\orgdiv{WINDY Lab}, \orgname{ Department of Artificial Intelligence, Westlake University}, \orgaddress{ \city{Hangzhou 310024}, \state{Zhejiang}, \country{China}}}

\affil[3]{\orgdiv{Westlake Institute for Optoelectronics}, \orgaddress{ \city{Hangzhou 311421}, \state{Zhejiang}, \country{China}}}

\affil[4]{\orgdiv{Future Industrial Research Centre, Westlake University}, \orgaddress{ \city{Hangzhou 310024}, \state{Zhejiang}, \country{China}}}

\affil[5]{\orgdiv{Institute of Advanced Technology, Westlake Institute for Advanced Study}, \orgaddress{ \city{Hangzhou 310024},  \state{Zhejiang}, \country{China}}}


\abstract{Aerial manipulators, composed of multirotors and robotic arms, have a structure and function highly reminiscent of avian species. This paper studies the tracking control problem for aerial manipulators. We propose an avian-inspired aerial manipulation system, which includes an avian-inspired robotic arm design, a Recursive Newton-Euler (RNE) method-based nonlinear flight controller, and a coordinated controller with two modes. Compared to existing methods, our proposed approach offers several attractive features. 
First, the morphological characteristics of avian species are used to determine the size proportion of the multirotor and the robotic arm in the aerial manipulator. Second, the dynamic coupling of the aerial manipulator is addressed by the RNE-based flight controller and a dual-mode coordinated controller. Specifically, under our proposed algorithm, the aerial manipulator can stabilize the end-effector’s pose, similar to avian head stabilization. The proposed approach is verified through three numerical experiments. 
The results show that even when the quadcopter is disturbed by different forces, the position error of the end-effector achieves millimeter-level accuracy, and the attitude error remains within 1$^{\circ}$. The limitation of this work is not considering aggressive manipulation like that seen in birds. Addressing this through future studies that explore real-world experiments will be a key direction for research.}




\maketitle

\clearpage

\section{Introduction}

An aerial manipulator is typically composed of a multirotor and a robotic arm.
It combines the rapid mobility of multirotors with the high-precision manipulation capabilities of manipulators, offering significant potential for various applications. Currently, aerial manipulators have been widely used in fields such as contact-based inspection \cite{jimenez2015aerial,bodie2020active}, aerial pick-and-place \cite{cao2023motion,ramon2020grasp}, tree cavity inspection \cite{steich2016tree}, aerial additive manufacturing \cite{zhang2022aerial} (see recent surveys in  \cite{ollero2021past} and \cite{meng2022aerial}).

Aerial manipulators exhibit similarities with avian species. Structurally, both have a flying body and limbs for executing operations. Functionally, both can fly in 3D space and perform precise manipulation. Consequently, existing research approaches toward aerial manipulators can be categorized into two main groups. The first involves traditional methods derived from studies on multirotors and robotic arms. The second encompasses biomimetic approaches inspired by observations of animals.

The first group aims to develop high-precision aerial manipulators. An aerial manipulator is a complex multibody system with a strong coupling between the multirotor and the robotic arm. The motion of the multirotor and the robotic arm influence each other \cite{ruggiero2015multilayer}. 
To achieve high precision control, aerial manipulators have been studied in motion control, coordinated control, and mechanical design. 

Motion control algorithms are explored to achieve high-precision control for aerial manipulators. Existing motion control methods for aerial manipulators are mainly divided into two categories: full-body and decoupled control. Full-body control methods treat the multirotor and the robotic arm as an integrated whole. Dynamic coupling is naturally considered within the full-body control methods. This approach is promising in terms of control accuracy and robustness if the nonlinear dynamics can be precisely modeled \cite{xilun2019review}. However, they are challenging to implement in practice because obtaining a precise nonlinear model is difficult, and the various measurements required by this approach are hard to achieve \cite{ollero2021past}.

\begin{figure}[!t]	
		\centering
		\includegraphics[width=\linewidth]{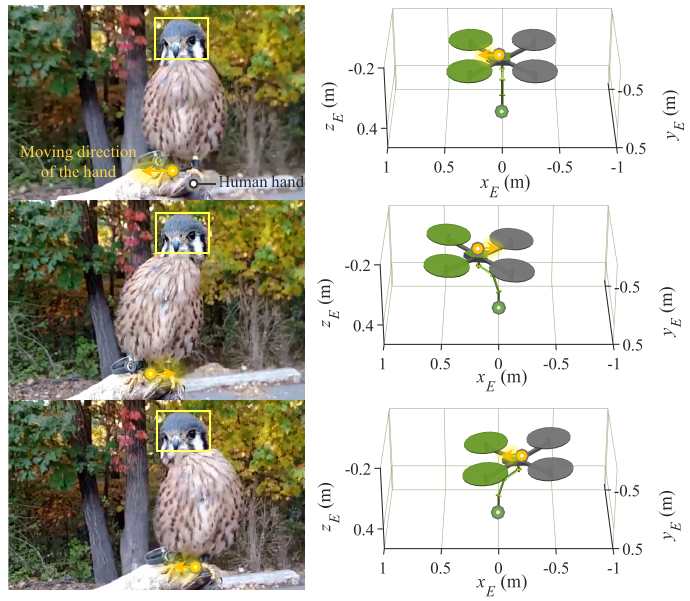}
	\caption{The visual description of the avian-inspired end-effector pose stabilization. Left: the head stabilization of a falcon\protect\footnotemark[1]. Right: the end-effector stabilization of the proposed aerial manipulator.}
	\label{fig_visual_description_avian_head}
\end{figure}
\footnotetext[1]{https://www.youtube.com/watch?v=JGArTWOJtXs}

Decoupled control methods treat the multirotor and the robotic arm as two separate entities and design the controller independently. In early works, the dynamic coupling between the multirotor and the robotic arm was ignored to simplify controller design \cite{Thomas_2014}. Existing control methods for multirotors and robotic arms were directly applied separately to each part. However, due to the influence of dynamic coupling, the control accuracy of this approach is relatively low. To achieve higher control accuracy, estimation methods have been introduced to estimate the forces and torques from the dynamic coupling \cite{cao2023eso, bodie2021dynamic, zhang2019robust}. Nevertheless, constrained by the physical performance of the aerial manipulator, these methods currently can only achieve control precision at the centimeter level.

The coordinated control allocates the motion of the quadcopter and the robotic to track the desired trajectory of the end-effector. It also has been studied to achieve higher precision end-effector control. In early works, closed-loop inverse kinematics (CLIK) was used, but it only utilized current state information without predicting future states \cite{baizid2017behavioral}. To improve accuracy, model predictive coordinated control methods have been proposed, leveraging predictive information of future states \cite{cao2020predictive, lunni2017nonlinear}. However, these methods do not consider tracking errors of the motion controller \cite{danko2015parallel}. Due to the kinematic characteristics of aerial manipulators, tracking errors of the multirotor are magnified at the end-effector.  To address this issue,  the tracking errors of the multirotor are considered in coordinated control, which enables end-effector errors to be maintained below 1 cm \cite{wang2023millimeter}. Although new coordinated control algorithms can improve precision, they are limited by the flight platform control accuracy and the mechanical characteristics of the robotic arm.

Mechanical design is another factor influencing the control precision. In early works, existing ground-mounted robotic arms were directly mounted onto multirotors to form aerial manipulators \cite{huber2013first}. However, unlike ground-mounted robotic arms, the motion of the robotic arm in aerial manipulators interferes with the motion of the multirotor. Using ground-based robotic arm solutions does not optimize the overall system performance. To reduce the impact of dynamic coupling, lightweight robotic arms have been employed in aerial manipulators \cite{suarez2017lightweight,bodie2021dynamic,meng2022aerial}. However, current robotic arm designs are task-specific. Determining the length of the robotic arm requires extensive engineering experience, as there is no unified criterion for determining the dimensions of the arm.

The second group aims to endow aerial manipulators with specific animal abilities. For example, to emulate the quick maneuvering and handling capabilities of animals, researchers developed a drone grasping system \cite{Thomas_2014}. Additionally, to mimic the perching ability of birds, bird-like leg mechanisms were designed \cite{hang2019perching,roderick2021bird}. Another example is the development of an aerial manipulator system with additive manufacturing capabilities to replicate the collective construction abilities of wasps \cite{zhang2022aerial}. These methods imitate animal capabilities found in nature. Nevertheless, drawing inspiration from nature is still essential for designing aerial manipulators. However, research in this area remains insufficient.

The above analysis reveals the limitations of the two groups of existing methods for aerial manipulators. This study aims to achieve avian-like manipulation capabilities in aerial manipulators from both mechanical design and algorithmic perspectives (see Fig.~\ref{fig_visual_description_avian_head}). Avian head stabilization is a fascinating phenomenon. Avians maintain the spatial stability of their eyes by adjusting the extension and rotation of their necks. This keeps their heads fixed relative to the ground during body movement, ensuring stable vision \cite{friedman1975visual,frost2009bird,katzir2001head}. Existing aerial manipulators have not yet achieved similar capabilities. To realize the end-effector stabilization inspired by avian head stabilization, we propose a framework that includes mechanical design, motion control, and coordinated control methods. Our framework can be applied to aerial inspection and grasping tasks, improving image acquisition quality and grasping success rates. The novelty of the proposed framework is summarized below:

 1) We propose a novel design framework inspired by avian morphology for the robotic arm in aerial manipulators. Due to the structural and functional similarities between aerial manipulators and avians, avian morphology can inform robotic arm design. We use the ratio of body width to neck length in avians to determine the size proportions between the multirotor and the robotic arm. Subsequently, existing robotic arm configurations are employed to finalize the robotic arm design. In the absence of a unified standard for designing robotic arm sizes in aerial manipulation, this method provides a new guiding framework for aerial manipulator design.

2) We propose a new partially decoupled motion control method based on the RNE approach. In this method, the quadcopter base and the robotic arm are controlled separately. The flight controller adopts a nonlinear control method, while the robotic arm is controlled using the computed torque approach. 
An estimator based on the Recursive Newton-Euler (RNE) method is employed to estimate the dynamic coupling force and torque between the quadcopter base and the robotic arm. This estimator leverages the known dynamics model to provide more accurate estimates of the coupling force and torque. By integrating the RNE-based estimator with our nonlinear motion controller, we ensure system stability despite the influence of dynamic coupling. Compared to the control method without dynamic coupling estimation, the inverse dynamic control method \cite{pierri2018adaptive}, the  Proportional-Integral-Derivative control (PID) control method \cite{2013PD}, and the geometric control method \cite{2021switchable}, our proposed controller demonstrates superior tracking accuracy.

3) A dual-mode coordinated control approach is proposed to achieve high-precision end-effector tracking. This coordinated control method allocates motion between the quadcopter base and the robotic arm, comprising two modes: hover mode and cooperation mode. The hover mode is employed when the end-effector's motion range is limited, with the quadcopter base maintaining a hover state. The cooperation mode is activated when the end-effector's motion range is extensive, requiring both the quadcopter and robotic arm to jointly track the desired trajectory of the end-effector. Under this coordinated control algorithm, the robotic arm's motion can compensate for the quadcopter's position and attitude errors. This ensures stabilization of the end-effector pose, similar to avian head stabilization.

Finally, we validate the proposed robotic arm, the RNE-based nonlinear decoupled control scheme, and the coordinated control through three numerical experiments.

The remainder of this paper is structured as follows. The mathematical model and the control system overview used in this paper are given in Section~\ref{sec_model}. The manipulator design and error analysis are presented in Section~\ref{sec_mani_design}. Section~\ref{sec_flight_control} proposes the flight controller for the quadcopter base. The manipulator and Coordinated control is proposed in Section~\ref{sec_Manipulator_Coordinate}. Then, the experimental verification for the proposed methods is given in Section~\ref{sec_experiment}. Conclusions are drawn in Section~\ref{sec_conclusion}.

\section{Mathematical Model and Control System Overview}\label{sec_model}

This section introduces the mathematical models of the aerial manipulator and the overall control framework.

\begin{figure*}[!t]
		\centering
		\includegraphics[width=\textwidth]{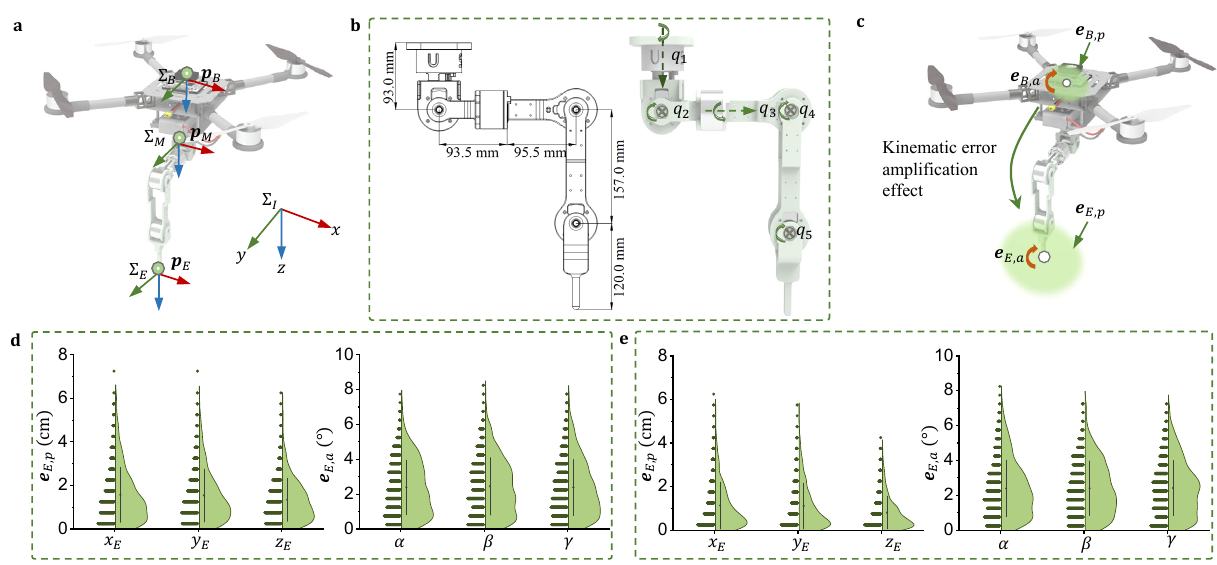}
 \caption{The coordinate frames, the visual depiction of the manipulator, and the illustration of the error amplification effect. \textbf{a}, The coordinate frames of the aerial manipulator. \textbf{b}, The manipulator's dimensions and configuration. \textbf{c}, Depiction of the error amplification effect. \textbf{d}, The distribution of end-effector position and attitude errors in the presence of base position and attitude errors. \textbf{e}, The distribution of end-effector position and attitude errors when only the base attitude error is present.}
  \label{am}
\end{figure*}

\subsection{Kinematics of The Aerial Manipulator}

As depicted in Fig.~\ref{am}a, the designed aerial manipulator consists of a quadcopter base and a 5-DoF robotic arm. This paper primarily utilizes four coordinate frames: the north-east-down (NED) inertial frame $\Sigma_I$, the body-fixed frame of the quadcopter $\Sigma_B$, the robotic arm frame $\Sigma_M$, and the end-effector frame $\Sigma_E$.


Let $\bm{R}_B \in SO(3)$ denote the rotation matrix of the quadcopter, while $\bm{\Phi}_B = [\phi, \theta, \psi]^T \in \mathbb{R}^3$ represents the attitude angle vector of the quadcopter. The quadcopter's position and the end-effector position are denoted by $\bm{p}_B, \bm{p}_E \in \mathbb{R}^3$, respectively. The end-effector position in $\Sigma_B$ is denoted by $\bm{p}_E^B \in \mathbb{R}^3$, which can be obtained through the forward kinematics of the manipulator \cite{edition2005introduction}. The relationship between $\bm{p}_E$ and $\bm{p}_E^B$ is expressed as
\begin{equation} 
\label{eq_pe}
	\bm{p}_E = \bm{p}_B + \bm{R}_B \bm{p}_E^B.
\end{equation}

The attitude angle vector of the end-effector is denoted by $\bm{\Phi}_E = [\alpha, \beta, \gamma]^T \in \mathbb{R}^3$. The rotation matrix of the end-effector is calculated as $\bm{R}_E= \bm{R}_x(\alpha)\bm{R}_y(\beta)\bm{R}_z(\gamma)$. Let $\bm{R}_E^B \in SO(3)$ denote the rotation matrix of the end-effector respective to $\Sigma_B$, which can be obtained through the forward kinematics of the robotic arm. The relationship between $\bm{R}_E$ and $\bm{R}_E^B$ is expressed as $\bm{R}_E = \bm{R}_B \bm{R}_E^B$.

Let $\bm{v}_E \in \mathbb{R}^3$ and $\bm{\omega}_E \in \mathbb{R}^3$ denote the velocity and the angular velocity of the end-effector, respectively. Similarly, $\bm{v}_B \in \mathbb{R}^3$ and $\bm{\omega}_B \in \mathbb{R}^3$ represent the quadcopter's velocity and angular velocity. The joint angle vector is denoted by $\bm{q}=[q_1,q_2,q_3,q_4,q_5]^T \in \mathbb{R}^5$. Let $\bm{p}_i \in \mathbb{R}^3$ denote the position of the origin of the $i$-th link frame. These link frames are obtained following the modified Denavit-Hartenberg (MDH) method \cite[Chapter 3]{edition2005introduction}. Let $\bm{z}_i$ denote the direction vector of the  $i$-th joint axis in  $\Sigma_I$. By differentiating \eqref{eq_pe}, the velocity of the end-effector is  
\begin{equation} 
\label{eq_ve}
	\bm{v}_E = \bm{v}_B - [\bm{R}_B \bm{p}_E^B]_{\times} \bm{\omega}_B + \bm{J}_v \dot{\bm{q}},
\end{equation}
where $[\cdot]_{\times}$ denotes the skew-symmetric matrix of a vector, $\bm{J}_v = [\bm{J}_1, \bm{J}_2, \dots, \bm{J}_i, \dots, \bm{J}_5]$, where $\bm{J}_i = \bm{z}_i \times (\bm{p}_E - \bm{p}_i), i \in 1, 2, \dots, 5$.
By differentiating $\bm{R}_E$, the angular velocity of the end-effector can be formulated as 
\begin{equation} 
\label{eq_oe}
	\bm{\omega}_E = \bm{\omega}_B + \bm{J}_o\dot{\bm{q}},
\end{equation}
where, $\bm{J}_o=[\bm{z}_1,\bm{z}_2,\bm{z}_3,\bm{z}_4,\bm{z}_5]$.

\subsection{Dynamics of the Quadcopter Base}
Let $ m_B $ and $ m_M $ be the mass of the quadcopter and the robotic arm, respectively. Then the total mass of the system is $m_S=m_B+m_M$.
The inertia tensor of the quadcopter in $\Sigma_B $ is denoted by $\bm{I}_B \in \mathbb{R}^{3 \times 3} $. The thrust is denoted by $ f $. The torque in $\Sigma_B $ is represented by $\bm{\tau}_B \in \mathbb{R}^3 $. 
The dynamics of the quadcopter base considers the dynamic coupling force $\bm{f}_D \in \mathbb{R}^3$ and torque $\bm{\tau}_D \in \mathbb{R}^3$. The external force and torque are denoted by $\bm f_{ext}$ and $\bm \tau_{ext}$, respectively.
According to \cite{cao2023eso}, the dynamics of the quadcopter base are
\begin{equation} 
\label{eq_dynamics_quad}
	\begin{aligned}
        & \dot{\bm{p}}_B = \bm{v}_B,\\
		&  \dot{\bm{v}}_B = -\frac{f}{m_S} \bm{R}_B \bm{e}_3 + g \bm{e}_3 + \frac{\bm{f}_D+ \bm f_{ext}}{m_S},\\
        & \dot{\bm{\Phi}}_B = \bm{Q}^{-1} \bm{\omega}_B^B,\\
 &  \dot{\bm{\omega}}_B^B = \bm{I}_B^{-1} (\bm{\tau}_B + \bm{\tau}_D^B+\bm \tau_{ext} - \bm{\omega}_B^B \times \bm{I}_B \bm{\omega}_B^B),
	\end{aligned}
\end{equation}
where $ \bm{e}_3=[0,0,1]^T$, $ g $ is the gravity acceleration, and $\bm{Q}$ is the transformation matrix from Euler angle rates to angular velocity, which is expressed as
\begin{equation}
    \begin{aligned}
    \bm{Q} =
    \begin{bmatrix}
        1 & 0 & -\sin(\theta) \\
        0 & \cos(\phi) & \cos(\theta) \sin(\phi) \\
        0 & -\sin(\phi) & \cos(\theta) \cos(\phi)
    \end{bmatrix}.
\end{aligned}
\end{equation}

The mass $ m_M $, the dynamic coupling force $\bm{f}_D$, and the torque $\bm{\tau}_D^B$ are contained in model \eqref{eq_dynamics_quad}. 
In the absence of the robotic arm,  we have $ m_M=0$, $\bm{f}_D=0$, and $\bm{\tau}_D^B=0$. Thus, the model \eqref{eq_dynamics_quad} degenerates to the dynamics of a standard quadcopter model.

\subsection{Control System Overview}

\begin{figure*}[!t]
	\centering
	\includegraphics[width=\textwidth]{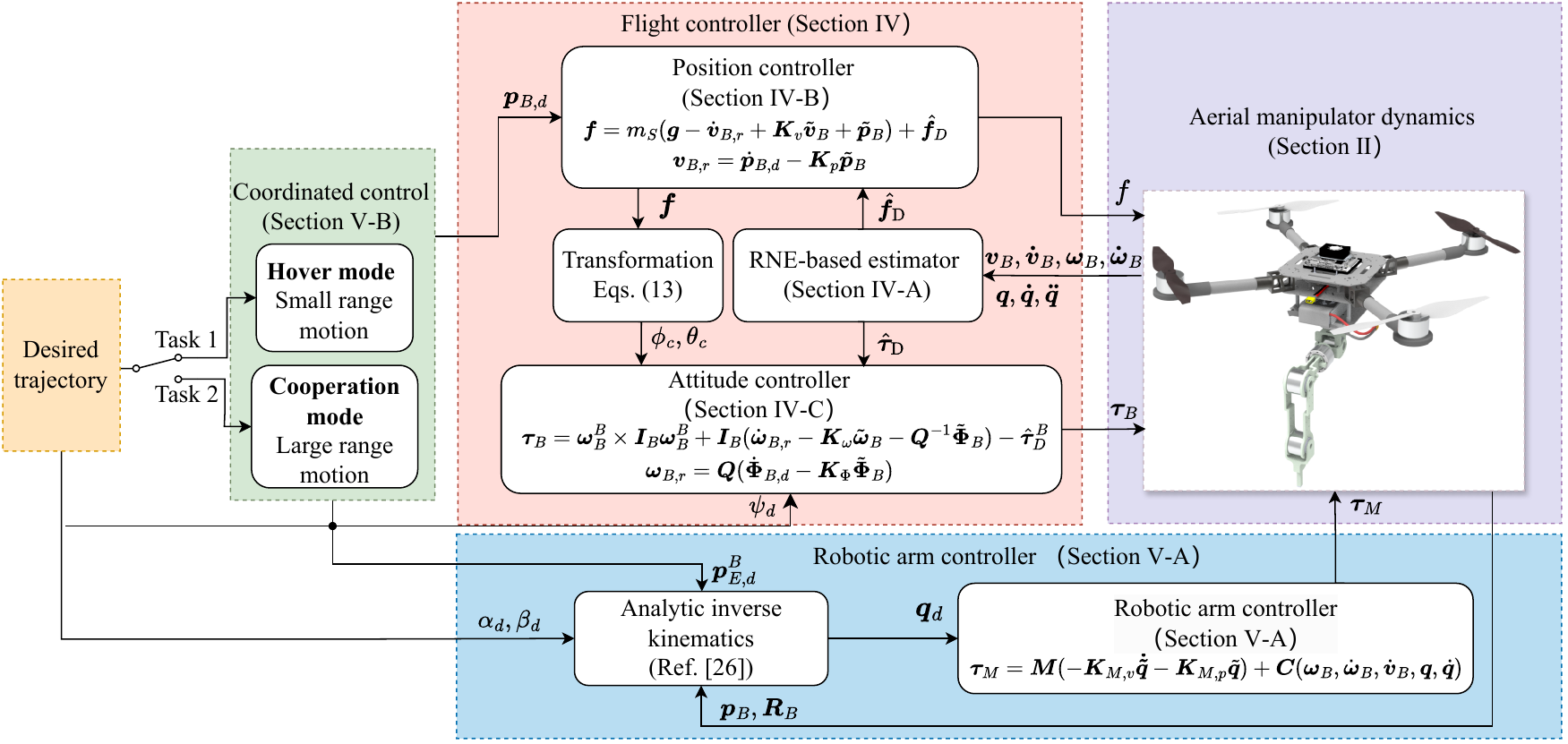}
	\caption{The proposed control scheme of the aerial manipulator system.}
	\label{control}
\end{figure*}
 The overall architecture of the control system is shown in Fig.~\ref{control}.  The system is decomposed into three components.
                
The first component is coordinated control. Coordinated control aims to produce the desired trajectories for both the quadcopter base and the robotic arm, thereby achieving accurate end-effector trajectory tracking.
Especially, it can maintain the end-effector at a fixed pose with high precision, similar to avian head stabilization. Coordinated control has two modes: hover mode and cooperation mode. Hover mode is suitable for small end-effector motion ranges, while cooperation mode is for larger ones.

The second component is the flight control. Its input is the desired trajectory of the quadcopter planned by the coordinated control. Its outputs are thrust commands $f$ and torque commands $\bm{\tau}_B$. This component can be further divided into two sub-components. The first utilizes the RNE method to estimate the dynamic coupling forces $\bm{f}_D$ and torque $\bm{\tau}_D$. The second is the nonlinear position and attitude control for the quadcopter base. The position control aims to generate the thrust commands $f$ to track the desired position $\bm{p}_{B,d}$, while the attitude control aims to generate the torque commands $\bm{\tau}_B$ to track the desired attitude $\bm{\Phi}_{B,d}$. 

The third component is the robotic arm control. Its inputs are the desired joint angles $\bm{q}_d$ and the quadcopter's velocities and accelerations. Its output is the joint torque command $\bm{\tau}_M \in \mathbb{R}^5$. 
This component consists of two parts. The first is the analytical inverse kinematics algorithm, which can rapidly calculate the desired joint angles given the known poses of the quadcopter and end-effector. The second is the computed torque controller, which ensures the asymptotic stability of the closed-loop dynamics for the robotic arm.

\section{Robotic Arm Design and Error Analysis} \label{sec_mani_design}
In this section, an avian-inspired design method is proposed for the robotic arm. Additionally, we analyze the workspace of the designed robotic arm and conduct a kinematic error analysis.
\subsection{Robotic Arm Design}

\begin{figure}[!t]
	\centering
	\includegraphics[width=\linewidth]{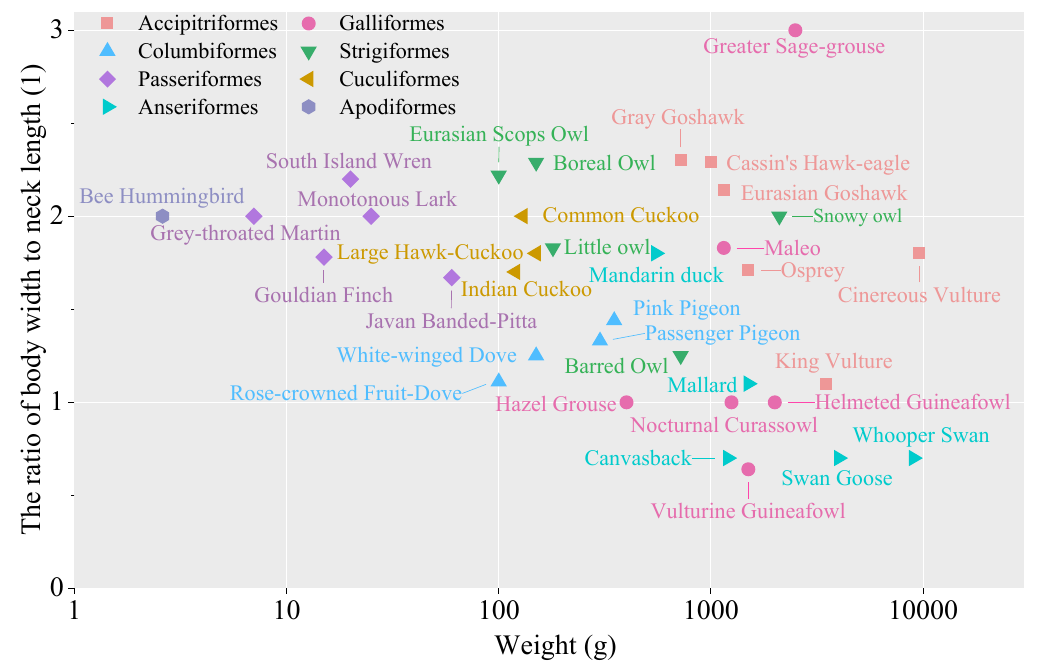}
	\caption{The ratio of body width to neck length for various birds. These birds come from 35 species across 8 orders. The horizontal axis uses a logarithmic scale to represent the average weight of bird populations for each species.}
	\label{avian_size}
\end{figure}

Currently, there is no unified guideline for designing robotic arms in aerial manipulators. The first approach involves directly using ground-mounted robotic arms, but it cannot ensure optimal overall system performance. The second approach designs robotic arms for specific tasks but requires extensive engineering experience. Unlike these approaches, we propose an avian-inspired design method for the robotic arm in aerial manipulators.

We determine the total length of the robotic arm based on avian morphological features. Specifically, the bird's body shares similarities with the quadcopter base, and the bird's neck is analogous to the robotic arm. Therefore, we derive the ratio of the quadcopter's wheelbase to the robotic arm's total length from the ratio of body width to neck length in birds. This process involves two steps. 
First,  we collect 88 frontal and lateral images of various birds from 35 species across 8 orders, sourced from the eBird website\footnotemark[2].
\footnotetext[2]{https://ebird.org/home.}
Second, we calculated the ratio by measuring pixel sizes in the images. The results are shown in Fig.~\ref{avian_size}. 
It can be seen that the weight of these birds ranges from 2~g to 10000~g, and the ratio ranges from 0.6 to 3. 
Among all avian species, most in the order Accipitriformes exhibit exceptional head stabilization, high agility, and strong maneuverability. Additionally, their body mass is comparable to that of quadcopters. Therefore, in this paper, we use the osprey's proportions to determine the ratio of the quadrotor's wheelbase to the total length of the robotic arm. This ratio is set at 1.7.
The wheelbase of the quadcopter used in this paper is 0.93~m. Therefore, the total length of the robotic arm is set as 0.55~m.

The geometric dimensions of the robotic arm are subsequently designed for aerial manipulators. First, we determine that the robotic arm will have five joints. In most tasks, there is no requirement for the end-effector attitude angle $\gamma$, as it can be provided by the quadcopter's yaw angle. Therefore, only five degrees of freedom (DoFs) of the end-effector need to be controlled, necessitating five joints for the robotic arm.
Second, the arrangement of these joints is established. As shown in Fig. \ref{am}b, the axes of the first three joints intersect at a single point, while the axes of the fourth and fifth joints are parallel. This configuration allows for a closed-form kinematic solution due to the three consecutive intersecting axes \cite{edition2005introduction}.
Third, the structural dimensioning of the links is designed to achieve the desired workspace for the robotic arm. In this paper, the desired workspace is designed to encompass a hemisphere with a radius of 0.5~m. Initially, the length of the first link is set to zero, while the remaining four links are assigned equal lengths that sum up to the total length of the robotic arm. These initial dimensions are then iteratively adjusted to ensure that the end-effector can reach all desired points within the hemisphere. The iterative process consists of two steps. The first step is to determine the workspace boundaries for the current structure size using the method described in \ref{workspace analysis}. The second step is to adjust each link's length accordingly. If there are internal points within the hemisphere that the end-effector cannot reach, the length of the last two links is increased and the length of the first three links is reduced. This process continues until the end-effector can successfully reach all the desired points within the hemisphere.
The final structural dimensioning of this robotic arm is illustrated in Fig.~\ref{am}b.

\begin{remark}
    Our proposed framework is for designing the robotic arm in aerial manipulators. The framework is not limited to the robotic arm configuration used in this paper. This framework can also be applied to other configurations.
\end{remark}

\subsection{Workspace analysis}\label{workspace analysis}

\begin{figure*}[!t]	
		\centering
		\includegraphics[width=\textwidth]{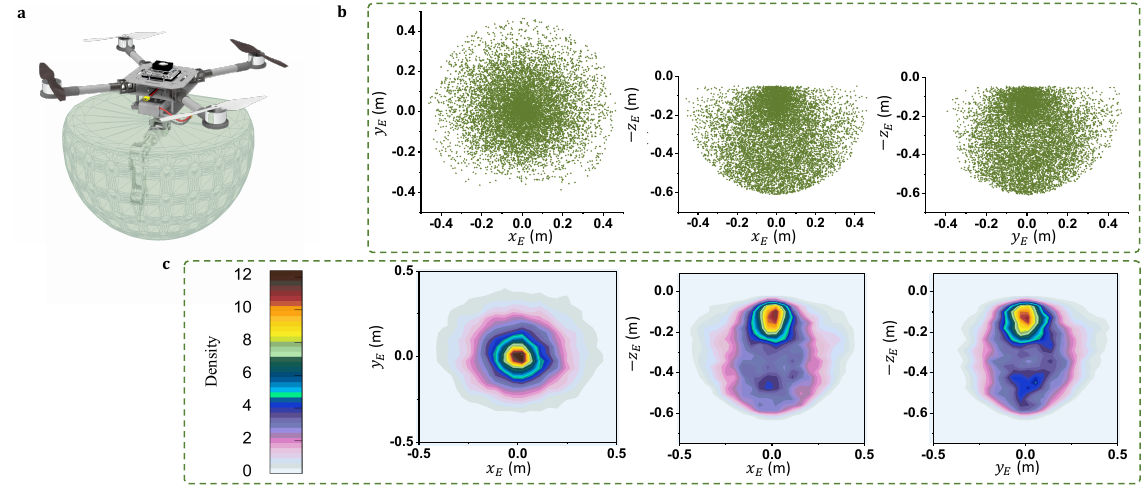}
	\caption{The workspace of the end-effector. \textbf{a}, The visual diagram of the workspace. \textbf{b}, Samples in \textit{x-y}, \textit{x-z}, and \textit{y-z} planes respectively. \textbf{c}, The position probability distribution via kernel density estimation.}
 	\label{workspace}
\end{figure*}

The workspace is a specification of the configurations that the end-effector of the robotic arm can reach. It is important for the control and planning of the robotic arm. It can be determined through numerical sampling in the configuration space. 
To avoid self-collision, the range of each joint is set as follows: $q_1, q_3, q_4, q_5$ are limited from $-3\pi/4$ to $3\pi/4$, while $q_2$ is restricted from $0$ to $3\pi/4$.
In numerical sampling, each joint is sampled following a uniform distribution within the corresponding joint limits. In this paper, ten thousand samples in the configuration space are selected to describe the workspace. 

The resulting samples for the robotic arm are illustrated in  Fig.~\ref{workspace}a,b. Kernel density estimation is employed to identify regions within the workspace that have a high probability of containing valid configurations. 
A Gaussian kernel function is used as the kernel function in this paper. The estimation result is given in Fig.~\ref{workspace}c.

\subsection{Kinematic Error Amplification Effect}\label{amplification}
The tracking errors of the quadcopter affect that of the end-effector through the forward kinematics \cite{10339889,wang2023millimeter}. Therefore, we introduce a Monte Carlo method to further quantitatively analyze the relationship between the tracking errors of the quadcopter and the end-effector. 

The right-hand subscript, e.g., $\bm{p}_{B,d}$, denotes the desired value of $\bm{p}_B$. Let $\bm{e}_{B,p}=\bm{p}_{B}-\bm{p}_{B,d} \in \mathbb{R}^3$ represent the error between the actual and desired positions of the quadcopter. According to \eqref{eq_pe}, we have  $\bm{p}_E=\bm{p}_{B,d}+\bm{e}_{B,p}+\bm{R}_B\bm{p}_E^B$.
Then, the end-effector's position error is 
\begin{equation}
\label{eq_e_pe}
        \bm{e}_{E,p}=\bm{p}_{E}-\bm{p}_{E,d}=\bm{e}_{B,p}+(\bm{R}_B-\bm{R}_{B,d})\bm{p}_E^B.
\end{equation}
In addition, the end-effector's rotation error can be calculated as 
\begin{equation}
\label{eq_e_re}
	\bm{E}_{E,a}=(\bm{R}_{E,d})^T \bm{R}_E=(\bm{R}_E^B)^T \bm{R}_{B,d}^T \bm{R}_B\bm{R}_E^B.
\end{equation}
Denote $\bm {e}_{E,a} \in \mathbb{R}^3$ as the attitude error vector of the end-effector. According to the definition of the rotation matrix, $\bm{e}_{E,a}$ can be calculated from $\bm {E}_{E,a}$. 

In the Monte Carlo method, the position error of the quadcopter follows a uniform distribution within the range [-0.2, 0.2]~m. The attitude error of the quadcopter follows a uniform distribution within the range [-5, 5]$^\circ$. The attitude angles of the quadcopter and the joint angles also follow a uniform distribution within the corresponding ranges, respectively. The sample size of the Monte Carlo method is 1000. The statistical results of these samples are illustrated in violin plots (see Fig.~\ref{am}d,e).

One can obtain the following conclusions from Fig.~\ref{am}d,e. The forward kinematics of the aerial manipulator significantly increases the end-effector's position error ( consistent with \eqref{eq_e_pe} ). The mean position error of the quadcopter is 1.91 cm, while the mean position error of the end-effector is 2.96 cm, amplified by about 1.5 times. With respect to the attitude error of the end-effector, there is no significant amplification effect here (both are 4.8$^{\circ}$). This is consistent with \eqref{eq_e_re}. 
Further analysis of \eqref{eq_e_pe} reveals that when only the quadcopter's position error exists, the end-effector's position error equals the quadcopter's, without any amplification effect.
However, when only the quadcopter's attitude error exists, the mean position error of the end-effector remains almost unchanged, but the standard deviation is amplified by 2.4 times (see Fig.~\ref{am}e).
These conclusions illustrate the kinematic amplification effect. However, \eqref{eq_e_pe} indicates once the arm's workspace is determined (i.e., the range of $\bm{p}_{E}^B$ is determined),  position errors of the quadcopter cannot be eliminated through mechanical design alone. Thus, we must focus on controller design to address this issue.

\section{Flight Controller Design}\label{sec_flight_control}
This section is divided into three parts: RNE-based dynamic coupling estimation, position control, and attitude control of the quadcopter base.

\subsection{RNE-Based Dynamic Coupling Estimation} \label{RNE}
The RNE method \cite[Section 6.5]{edition2005introduction} is adopted to estimate the dynamic coupling force $\bm{f}_D$ and torque $\bm{\tau}_D$. 
The estimating process consists of two steps.

The first step is the forward process, which computes the linear and angular acceleration of each link's center of mass (CoM). These computations are performed iteratively, starting with link 0 and progressing outward, link by link, to link $n$.
Consider the quadcopter as link 0. Let $\bm{p}_{M}^B \in \mathbb{R}^3$ represent the position of origin of the frame $\Sigma_M$ in $\Sigma_B$. For link 0, we have
\begin{equation}
    \begin{aligned}
\bm{\omega}_0^0&=\bm{\omega}_B^B,\ 
\dot{\bm{\omega}}_0^0=\dot{\bm{\omega}}_B^B,\ 
\bm{p}_{M}^0=\bm{p}_{M}^B,\\
\dot{\bm{v}}_0^0&=\bm{R}_B^T(\dot{\bm{v}}_B-[0,0,g]^T)+\dot{\bm{\omega}}_0^0\times \bm{p}_{M}^0 \\&+ \bm{\omega}_0^0 \times (\bm{\omega}_0^0 \times \bm{p}_{M}^0).
\end{aligned}
\end{equation}

Let $\bm{R}_{i}^{i+1}$ denote the rotation matrix from $\Sigma_i$ to $\Sigma_{i+1}$.
$\bm{z}=[0,0,1]^T$ is a unit vector.
The angular velocity $\bm{\omega}_{i+1}^{i+1}$ and the angular acceleration $\dot{\bm{\omega}}_{i+1}^{i+1}$ of link $i+1$ are calculated as
\begin{equation}
    \begin{aligned}
    \bm{\omega}_{i+1}^{i+1} &= \bm{R}_i^{i+1}\bm{\omega}_{i}^i+\dot{q}_{i+1} \bm{z},\\
    \dot{\bm{\omega}}_{i+1}^{i+1} &= \bm{R}_i^{i+1}\dot{\bm{\omega}}_{i}^i+\ddot{q}_{i+1} \bm{z}+\dot{q}_{i+1} (\bm{R}_i^{i+1}\bm{\omega}_i^i) \times \bm{z}.
\end{aligned}
\end{equation}

Let $\bm{p}_{i+1}^i \in \mathbb{R}^3$ represent the position of the origin of $\Sigma_{i+1}$ in $\Sigma_{i}$, and 
$\bm{p}_{C_i}^i \in \mathbb{R}^3$ represent the position of the CoM of link $i$.
The linear acceleration of the origin and the CoM of link $i+1$ are calculated as
\begin{equation}
    \begin{aligned}
    \dot{\bm{v}}_{i+1}^{i+1} &= \bm{R}_i^{i+1}(\dot{\bm{v}}_i^i+\dot{\bm{\omega}}_i^i\times \bm{p}_{i+1}^i+\bm{\omega}_i^i\times(\bm{\omega}_i^i \times \bm{p}_{i+1}^i)),\\
    \dot{\bm{v}}_{C_{i+1}}^{i+1} &= \dot{\bm{v}}_{i+1}^{i+1}+\dot{\bm{\omega}}_{i+1}^{i+1}\times \bm{p}_{C_{i+1}}^{i+1}\\
    &+\bm{\omega}_{i+1}^{i+1}\times(\bm{\omega}_{i+1}^{i+1}\times \bm{p}_{C_{i+1}}^{i+1}).
\end{aligned}
\end{equation}

The second step is the backward process, which computes the internal force and torque exerted on link $i$ by link $i-1$. These computations are performed link by link, starting with link $n$ and progressing inward toward the base of the robotic arm.
We calculate the net force $\bm{\zeta}_{i+1}^{i+1} \in \mathbb{R}^3$ and net torque $\bm{\chi}_{i+1}^{i+1} \in \mathbb{R}^3$ acting at the CoM of link $i+1$. They are given as
\begin{equation}
    \begin{aligned}
&\bm{\zeta}_{i+1}^{i+1}=m_{i+1}\dot {\bm{v}}_{C_{i+1}}^{i+1},\\
&\bm{\chi}_{i+1}^{i+1}=\bm{I}_{i+1}\dot {\bm{\omega}}_{i+1}^{i+1}+\bm{\omega}_{i+1}^{i+1}\times \bm{I}_{i+1}\bm{\omega}_{i+1}^{i+1},
\end{aligned}
\end{equation}
where $m_{i+1}$ represents the mass of link $i+1$, and $\bm{I}_{i+1} \in \mathbb{R}^{3\times3}$ represents the moment of inertia of link $i+1$. Let $\bm{f}_i^i \in \mathbb{R}^3$ and $\bm{\tau}_i^i \in \mathbb{R}^3$ represent force and torque exerted on link $i$ by link $i-1$ in $\Sigma_I$. They can be calculated by force-balance and torque-balance relationship:
\begin{equation}
\begin{aligned}
\bm{f}_i^i&=\bm{\zeta}_i^i+\bm{R}_{i+1}^i \bm{f}_{i+1}^{i+1},\\
\bm{n}_i^i&=\bm{\chi}_i^i+\bm{R}_{i+1}^i\bm{n}_{i+1}^{i+1}+\bm{p}_{C_i}^i\times \bm{\zeta}_i^i\\
&+\bm{p}_{i+1}^i\times (\bm{R}_{i+1}^i \bm{f}_{i+1}^{i+1}).
\end{aligned}
\end{equation}

In this paper, the end-effector is not in contact with the external environment, so the contact force and torque are zero, i.e., $\bm{f}_6^6=0$ and $\bm{\tau}_6^6=0$. 
The dynamic coupling force and torque exerted by the arm on the quadcopter at $\bm{p}_M$ are $-\bm{f}_1^1$ and $-\bm{n}_1^1$. Transforming the force $-\bm{f}_1^1$ into the $\Sigma_I$  and the torque $-\bm{n}_1^1$ into the $\Sigma_B$, we can obtain 
\begin{equation} 
\begin{aligned}
&\bm{f}_{D}=-\bm{R}_B\bm{R}_1^B\bm{f}_1^1-m_Mg\bm{e}_3,\\
&\bm{\tau}_{D}^B=-\bm{p}_{M}^B\times(\bm{R}^B_1 \bm{f}_1^1)-\bm{R}^B_1\bm{n}_1^1.
\end{aligned}
\label{eq_fd_tn}
\end{equation}
In \eqref{eq_fd_tn}, the gravity term $m_Mg\bm{e}_3$ is introduced since the gravity of the robotic arm has been considered in the model \eqref{eq_dynamics_quad}.

\subsection{Position control}\label{sec_positioncontrol}
The position error of the quadcopter is defined as $\tilde{\bm{p}}_B=\bm{p}_B-\bm{p}_{B,d}$. To depict the desired velocity of the quadcopter, we define 
\begin{equation}
	\bm{v}_{B,r}= \dot{\bm{p}}_{B,d} - \bm{K}_p \tilde{\bm{p}}_B,
	\label{eq_vr}
\end{equation}
where $\bm{K}_p\in\mathbb{R}^{3\times 3}$ is a positive diagonal matrix. Then, the velocity error of the quadcopter is defined as $\tilde{\bm{v}}_B=\dot{\bm{p}}_B-\bm{v}_{B,r}$.
Let $\hat{\bm{f}}_{D}$ denote the estimate of the dynamic coupling force. Then, the position controller of the quadcopter is defined as 
\begin{equation}
	\bm{f}=m_S\left(\bm{g} - \dot{\bm{v}}_{B,r}  + \bm{K}_v\tilde{\bm{v}}_B+ \tilde{\bm{p}}_B+ \hat{\bm{f}}_{D}\right),
	\label{eq_posi_cont}
\end{equation}
where  $\bm{K}_v\in\mathbb{R}^{3\times 3}$ is a positive diagonal matrix.
\begin{theorem}\label{theorem_posi}
	Assuming that the estimated error of the dynamic coupling force and the external disturbance force is bounded, i.e., $\|{\bm{f}}_{D} - \hat{\bm{f}}_{D}\| \leq \delta_{f,D}$ and $\|\bm f_{ext}\| \leq \delta_{f, ext}$, if the position controller is designed as \eqref{eq_posi_cont} with the proposed RNE-based dynamic coupling estimator, then the position error $\tilde{\bm{p}}_B$ is bounded, i.e.,  $\| \tilde{\bm{p}}_B\|\leq ({\delta_{f,D}}+\delta_{f, ext} )/{\lambda_p}$ as $t\rightarrow\infty$, where $ {\delta_{f,D}}$, $\delta_{f, ext}$ and ${\lambda_p}$ are positive constant parameters.
\end{theorem}
\begin{proof}
	A candidate Lyapunov function is given as 
	\begin{equation}
			V_{p}(\bm{r}_p)=\frac{1}{2}\tilde{\bm{p}}_B^T\tilde{\bm{p}}_B+\frac{1}{2}\tilde{\bm{v}}_B^T\tilde{\bm{v}}_B,
	\end{equation}
  where $\bm{r}_p=[\tilde{\bm{p}}_B^T,\tilde{\bm{v}}_B^T]^T$. The time derivative of $V_p$ is given as 
  \begin{equation}
\begin{aligned} 
\dot{V}_p=&\tilde{\bm{p}}_B^T\dot{\tilde{\bm{p}}}_B+\tilde{\bm{v}}_B^T\dot{\tilde{\bm{v}}}_B \\
		=&\tilde{\bm{p}}_B^T(\dot{\bm{p}}_B-\dot{\bm{p}}_{B,d}+\bm{K}_p\tilde{\bm{p}}_B-\bm{K}_p\tilde{\bm{p}}_B)+\tilde{\bm{v}}_B^T\dot{\tilde{\bm{v}}}_B \\
		=& -\tilde{\bm{p}}_B^T\bm{K}_p\tilde{\bm{p}}_B+ \tilde{\bm{p}}_B^T\tilde{\bm{v}}_B\\
  &+\tilde{\bm{v}}_B^T(\bm{f}_{D} - \hat{\bm{f}}_{D} +\bm f_{ext}-\tilde{\bm{p}}_B-\bm{K}_v\tilde{\bm{v}}_B)\\
		=& -\tilde{\bm{p}}_B^T\bm{K}_p\tilde{\bm{p}}_B -\tilde{\bm{v}}_B^T\bm{K}_v\tilde{\bm{v}}_B\\
 & +\tilde{\bm{v}}_B^T(\bm{f}_{D} - \hat{\bm{f}}_{D}+\bm f_{ext}).
\end{aligned}
\end{equation}

Note that we assume $\|{\bm{f}}_{D} - \hat{\bm{f}}_{D}\| \leq \delta_{f,D}$ and $\|\bm f_{ext}\| \leq \delta_{f,ext}$. Using the comparison theorem \cite[Section~9.3]{khalil2002nonlinear}, we define $W_p=\sqrt{V_p}=\frac{1}{\sqrt{2}}\left\|\bm{r}_p\right\| $, where $\bm{r}_p=[\tilde{\bm{p}}_B^T,\tilde{\bm{v}}_B^T]^T$. After finite time $T_p$, we can obtain
\begin{equation}
    \begin{aligned} 
    \dot{W}_p&=\frac{-\tilde{\bm{p}}_B^T\bm{K}_p\tilde{\bm{p}}_B -\tilde{\bm{v}}_B^T\bm{K}_v\tilde{\bm{v}}_B+\tilde{\bm{v}}_B^T(\bm{f}_{D} - \hat{\bm{f}}_{D} +\bm f_{ext} ) }{\sqrt{2}\left\|\bm{r}_p\right\| }\\
    &\leq \frac{-\lambda_p\left\|\bm{r}_p\right\|^2+(\delta_{f,D}+\delta_{f,ext})\left\|\bm{r}_p\right\| }{\sqrt{2}\left\|\bm{r}_p\right\| }\\
    &=-\lambda_pW_p+\frac{\delta_{f,D} +\delta_{f,ext}}{\sqrt{2}},
\end{aligned}
\end{equation}

where $\lambda_p$ is the minimum eigenvalue of $\bm{K}_p$ and $\bm{K}_v$.  Since $\bm{K}_p$ and $\bm{K}_v$ are positive diagonal matrices, we have $\lambda_p>0$.

According to the BIBO stability \cite[Section~5.2]{khalil2002nonlinear}, we have $\left| W_p\right| \leq({\delta_{f,D}+\delta_{f,ext}})/(\sqrt{2}\lambda_p)$ as $t\rightarrow \infty$. With the definition of $W_p$, we have $\|\bm{r}_p\| \leq ({\delta_{f,D}} +\delta_{f,ext})/{\lambda_p}$ as $t\rightarrow \infty$. According to the definition of $\bm{r}_p$, we have $\| \tilde{\bm{p}}_B\|\leq ({\delta_{f,D}} +\delta_{f,ext})/{\lambda_p}$ as $t\rightarrow\infty$.
\end{proof}

 Consequently, the desired total thrust $f$ and desired attitude angles can be computed as 
 \begin{equation}
	\begin{aligned} 
 f &=\left\| \bm{f}\right\|,\\
\phi_c&=\text{asin}\left( \frac{f_1\sin\psi_d - f_2\cos\psi_d}{f} \right) ,\\
\theta_c&=\text{atan}\left( \frac{f_1\cos\psi_d + f_2\sin\psi_d}{f_3} \right), 
	\end{aligned}
\end{equation}
 where $f_i$ is the $i$-th element of $\bm{f}$.
 
\subsection{Attitude control}\label{sec_attcontrol}
The attitude error of the quadcopter is defined as $\tilde{\bm{\Phi}}_B=\bm{\Phi}_B-\bm{\Phi}_{B,d} \in \mathbb{R}^3$. To depict the desired angular velocity of the quadcopter, we define 
\begin{equation}
	\bm{\omega}_{B,r}=\bm{Q}(\dot{\bm{\Phi}}_{B,d}-\bm{K}_{\Phi}\tilde{\bm{\Phi}}_B),
	\label{eq_omega_r}
\end{equation}
where $\bm{K}_{\Phi}\in\mathbb{R}^{3\times 3}$ is a positive diagonal matrix. Then, the angular velocity error of the quadcopter is defined as $\tilde{\bm{\omega}}_B=\bm{\omega}_B^B-\bm{\omega}_{B,r}$.
Let $\hat{\bm{\tau}}_{D}$ denote the estimate of the dynamic coupling torque. The attitude controller of the quadcopter is given as 
\begin{equation}
\begin{aligned}
\bm{\tau}_B&=\bm{\omega}_B^B\times\bm{I}_B\bm{\omega}_B^B\\
&+\bm{I}_B(\dot{\bm{\omega}}_{B,r}-\bm{K}_{\omega}\tilde{\bm{\omega}}_B-\bm{Q}^{-1}\tilde{\bm{\Phi}}_B-\hat{\bm{\tau}}_{D}^B),
	\label{eq_cont_att}
 \end{aligned}
\end{equation}
where $\bm{K}_{\omega}\in\mathbb{R}^{3\times 3}$ is a positive diagonal matrix.

\begin{theorem}
	\label{theorem_att}
	Assuming that the estimated error of the dynamic coupling torque and the external disturbance torque is bounded, i.e., $\|{\bm{\tau}}_{D}^B - \hat{\bm{\tau}}_{D}^B\| \leq \delta_{\tau,D}$ and $ \|\bm \tau_{ext} \| \le\delta_{\tau, ext}$,  if the attitude controller is designed as \eqref{eq_cont_att} with the proposed RNE-based dynamic coupling estimator, then the attitude error $\tilde{\bm{\Phi}}_B$ is bounded, i.e., $\| \tilde{\bm{\Phi}}_B\|\leq ({\delta_{\tau,D}}+ \delta_{\tau,ext})/{\lambda_a}$  as $t\rightarrow\infty$, where $ {\delta_{\tau,D}}$, $\delta_{\tau, ext}$  and ${\lambda_a}$ are positive constant parameters.
\end{theorem}
\begin{proof}
	 A candidate function is given as 
	\begin{equation}
		V_a(\bm{r}_a)=\frac{1}{2}\tilde{\bm{\Phi}}_B^T\tilde{\bm{\Phi}}_B+\frac{1}{2}\tilde{\bm{\omega}}_B^T\tilde{\bm{\omega}}_B,
	\end{equation}
	where $\bm{r}_a=[\tilde{\bm{\Phi}}_B^T,\tilde{\bm{\omega}}_B^T]^T$. The time derivative of $V_a$ is given as 
	\begin{equation}
	\begin{split}
\dot{V}_a=&\tilde{\bm{\Phi}}_B^T\dot{\tilde{\bm{\Phi}}}_B+\tilde{\bm{\omega}}_B^T\dot{\tilde{\bm{\omega}}}_B\\
 =& \tilde{\bm{\Phi}}_B^T(\dot{\tilde{\bm{\Phi}}}_B-\dot{\tilde{\bm{\Phi}}}_{B,d}+ \bm{K}_{\Phi}\tilde{\bm{\Phi}}_B-\bm{K}_{\Phi}\tilde{\bm{\Phi}}_B)+\tilde{\bm{\omega}}_B^T\dot{\tilde{\bm{\omega}}}_B\\
 =&-\tilde{\bm{\Phi}}_B^T\bm{K}_{\Phi}\tilde{\bm{\Phi}}_B+ \tilde{\bm{\Phi}}_B^T\bm{Q}^{-1}{\tilde{\bm{\omega}}}_B \\
& + {\tilde{\bm{\omega}}}_B^T(\bm{\tau}_{D}^B - \hat{\bm{\tau}}_{D}^B  +\bm \tau_{ext}
-\bm{Q}^{-1}\tilde{\bm{\Phi}}_B - \bm{K}_{\omega}\tilde{\bm{\omega}}_B)\\
 =&-\tilde{\bm{\Phi}}_B^T\bm{K}_{\Phi}\tilde{\bm{\Phi}}_B -\tilde{\bm{\omega}}_B^T\bm{K}_{\omega}\tilde{\bm{\omega}}_B\\
 &+ \tilde{\bm{\omega}}_B^T({\bm{\tau}}_{D}^B - \hat{\bm{\tau}}_{D}^B +\bm \tau_{ext}).
	\label{eq_va_dot1}
 	\end{split}
	\end{equation}
Note that we assume $\|{\bm{\tau}}_{D}^B - \hat{\bm{\tau}}_{D}^B\| \leq \delta_{\tau,D}$ and $ \|\bm \tau_{ext} \|\le\delta_{\tau, ext}$. Using the comparison theorem \cite[Section~9.3]{khalil2002nonlinear}, we define $W_a=\sqrt{V_a}=\frac{1}{\sqrt{2}}\left\|\bm{r}_a\right\|$. After finite time $T_a$, we can obtain
\begin{equation}
\begin{split}
    \dot{W}_a & =\frac{-\tilde{\bm{\Phi}}_B^T\bm{K}_{\Phi}\tilde{\bm{\Phi}}_B -\tilde{\bm{\omega}}_B^T\bm{K}_{\omega}\tilde{\bm{\omega}}_B}{\sqrt{2}\left\|\bm{r}_p\right\|}\\
&+\frac{\tilde{\bm{\omega}}_B^T({\bm{\tau}}_{D}^B - \hat{\bm{\tau}}_{D}^B +\bm \tau_{ext})}{\sqrt{2}\left\|\bm{r}_p\right\|}\\
    &\leq \frac{-\lambda_a\left\|\bm{r}_a\right\|^2+(\delta_{\tau,D}  + \delta_{\tau,ext})\left\|\bm{r}_a\right\| }{\sqrt{2}\left\|\bm{r}_a\right\| }\\
    &=-\lambda_a W_a+\frac{\delta_{\tau,D}+ \delta_{\tau,ext}}{\sqrt{2}},
\end{split}
\end{equation}
where $\lambda_a$ is the minimum eigenvalue of $\bm{K}_{\Phi}$ and $\bm{K}_{\omega}$. Since $\bm{K}_{\Phi}$ and $\bm{K}_{\omega}$ are positive diagonal matrices, we have $\lambda_a>0$.

According to the BIBO stability \cite[Section~5.2]{khalil2002nonlinear}, we have $\left| W_a\right| \leq ({\delta_{\tau,D}} + \delta_{\tau,ext})/(\sqrt{2}\lambda_a)$ as $t\rightarrow \infty$. With the definition of $W_a$, we have $\|\bm{r}_a\| \leq ({\delta_{\tau,D}}+ \delta_{\tau,ext})/{\lambda_a}$ as $t\rightarrow \infty$. According to the definition of $\bm{r}_a$, we have $\| \tilde{\bm{\Phi}}_B\|\leq ({\delta_{\tau,D}}+ \delta_{\tau,ext})/{\lambda_a}$ as $t\rightarrow\infty$.
\end{proof}

 \section{Robotic Arm and Coordinated control}\label{sec_Manipulator_Coordinate}
This section first introduces the robotic arm control, followed by the coordinated control.

\subsection{Robotic Arm Control}
The computed torque control method \cite[Chapter 4]{murray2017mathematical} is introduced to design the robotic arm control. The desired joint angles are calculated by the analytical inverse kinematics algorithm, which is deduced by combining geometric and algebraic methods \cite[Chapter 4]{edition2005introduction}.
Let $\bm{q}_d \in \mathbb{R}^5$ denote the desired joint angle vector. The joint angle error is defined as $\tilde{\bm{q}}=\bm{q}-\bm{q}_d$. The computed torque controller is given by
\begin{equation}
    \begin{aligned}
    \bm{\tau}_M&= \bm{M}(-\bm{K}_{M,v} {\dot{\tilde{\bm{q}}}} -\bm{K}_{M,p}\tilde{\bm q})\\
    &+\bm{C}(\bm{\omega}_B,\dot {\bm{\omega}}_B,\dot {\bm{v}}_B,\bm q,\dot {\bm{q}}),
\end{aligned}
\end{equation}
where $\bm{M}$ is the inertia matrix, $\bm{C}$ is a nonlinear term, including centrifugal force terms, Coriolis acceleration terms, and gravity terms, and $\bm{K}_{M,v} \in \mathbb{R}^{3\times3}$ and $\bm{K}_{M,p} \in\mathbb{R}^{3\times3}$ are constant gain matrices.
The disturbances from the quadcopter base are included in the term $\bm C$, which is compensated for in the computed torque controller.
The error dynamics can be written as $\ddot{\tilde{\bm q}}+\bm{K}_{M,v} {\dot{ \tilde{\bm q}}}+\bm{K}_{M,p}\tilde{\bm q}=0$.  Since the error equation is linear, $\bm{K}_{M,v}$ and $\bm{K}_{M,p}$ can be easily chosen so that the overall system is exponentially stable \cite[Chapter 4]{murray2017mathematical}. 

 The stability analysis of the entire system is divided into two parts: stability of the quadcopter base and stability of the robotic arm. In Section \ref{sec_positioncontrol} and \ref{sec_attcontrol}, we have presented the stability proofs for the quadcopter base's position and attitude control. This section demonstrates the exponential stability of the robotic arm. Therefore, both subsystems of the entire system are stable.

\subsection{Coordinated control}\label{Coordinated control}

The coordinated control algorithm is designed to allocate the motion of the quadcopter base and the robotic arm to track the desired trajectory of the end-effector. 
Although the avian head stabilization dynamic model may provide valuable insights, we did not incorporate it into our coordinated control design for the following reason. Current literature primarily models the avian head stabilization dynamic as a mass-spring-damper system \cite{2015passive_avian_head}, where the bird's neck is represented as a spring with stiffness $k$, and the head is treated as a mass $m$. However, our robotic arm is a rigid structure with five degrees of freedom and cannot currently be equivalently modeled as a mass-spring-damper system. Nevertheless, we can achieve similar functionality through careful design of the coordinated control. Our coordinated control algorithm operates in two modes: hover mode and cooperation mode.

In the hover mode, the quadcopter hovers and maintains its position, which is suitable for scenarios where the end-effector's motion range is relatively small, such as pick-and-place tasks \cite{heredia2014control} and peg-in-hole tasks \cite{10339889}. The inputs of this mode are $\bm{p}_{E,d}$, ${\dot{\bm p}}_{E,d}$, $\alpha_d$, $\beta_d$, $\dot \alpha_d$, and $\dot \beta_d$. The outputs are $\bm{p}_E^B$ and ${\dot {\bm q}}_d$. From \eqref{eq_pe}, we have $\bm{p}_{E,d}^B=\bm{R}_B^T(\bm{p}_{E,d}-\bm{p}_B)$. 
Therefore, the desired joint angles $\bm q_d$ can be solved by the analytical inverse kinematics algorithm \cite[Chapter 4]{edition2005introduction}. 

The desired joint velocity is calculated based on the Jacobian matrix.
Let $\bm{\eta}_B=[\bm{v}_B^T,\bm{\omega}_B^T]^T\in \mathbb{R}^6$ represent the generalized velocity of the quadcopter base. Since we only control two attitude angles $\alpha$ and $\beta$, we define the generalized velocity of the end-effector as $\bm{\eta}_E=[\bm{v}_E^T,\dot{\alpha},\dot{\beta}]^T \in \mathbb{R}^5$. Let $\bm{T}\in \mathbb{R}^{2\times 3}$ denote the transformation matrix between $[\dot{\alpha},\dot{\beta}]^T$ and the angular velocity $\bm{\omega}_E$. It is formulated as
\begin{equation}
    \begin{aligned}
    \bm{T}=\begin{bmatrix}
        1 & \sin (\alpha)\tan (\beta) & -\cos (\alpha)\tan (\beta)\\
        0& \cos(\alpha) & \sin(\alpha)\\
    \end{bmatrix}.
\end{aligned}
\end{equation}

Let $\bm{E}_3$ denote the $3$-dimensional identity matrix. Combining \eqref{eq_ve} and \eqref{eq_oe}, we obtain $\bm{\eta}_E=\bm{J}_B \bm{\eta}_B+\bm{J}_q\dot{\bm{q}}$, where
\begin{equation}
    \begin{aligned}
\bm{J}_B=\begin{bmatrix}
    \bm{E}_3&-[\bm{R}_B\bm{p}_E^B]_{\times}\\
    \bm{0}_2&\bm{T}
\end{bmatrix},
\bm{J}_q=\begin{bmatrix}
    \bm{J}_v\\
   \bm{T} \bm{J}_o
\end{bmatrix}.
\end{aligned}
\end{equation}

Let $\bm{\eta}_{E,d} \in \mathbb{R}^5$ represent the desired generalized velocity of the end-effector. The desired joint velocity can be solved as
\begin{equation}
\label{eq_des_dq}
    \dot{\bm{q}}_d=\bm{J}_q^{\dagger} (\bm{\eta}_{E,d} - \bm{J}_B \bm{\eta}_B),
\end{equation}
where $\bm{J}_q^{\dagger}$ is the generalized inverse matrix of $\bm{J}_q$, given by $\bm{J}_q^{\dagger}= \bm{J}_q^T(\bm{J}_q\bm{J}_q^T)^{-1}$. 

The cooperation mode is suitable for scenarios where the end-effector has a large motion range. The quadcopter and robotic arm are required to move together to accomplish tasks. This mode can be used for trajectory tracking \cite{Tzoumanikas2020AerialMU}, pulling/pushing tasks \cite{7139968}, etc.
Its input is $\bm{p}_{E,d}$. Its outputs are $\bm{p}_{B,d}$ and ${\dot {\bm q}}_d$.
To ensure the manipulability and dexterity of the robotic arm, we let the end-effector stay at the center of the workspace. Let $\bm{p}_{C}^B\in \mathbb{R}^3$ represent the workspace center in $\Sigma_B$. The desired position of the quadcopter is then calculated as $\bm{p}_{B,d}=\bm{p}_{E,d}-\bm{R}_{B}\bm{p}_{C}^B$.
Subsequently, the desired joint angles $\bm q_d$ can be solved using the method outlined in the hover mode.

\section{Experimental Validation}\label{sec_experiment}

This section presents numerical experimental results to validate the proposed method. 

\subsection{Experimental Setup}

The wheelbase of the quadcopter is 0.93~m.  The quadcopter and robotic arm mass are 4.39~kg and 1.03~kg, respectively. The dimensions of the arm are shown in Fig.~\ref{am}b. 
The feedback data for the quadcopter's position and attitude are updated at a sampling frequency of 100 Hz and 200 Hz, respectively. The feedback data for the joint angles and velocities are updated at 200 Hz. Moreover, zero mean normally distributed measurement noise, with a standard deviation of $2 \ \times 10^{-2} \text{m/s}^2$ for accelerations of the quadcopter, $10^{-2} \ \text{rad/s}^2$ for angular accelerations
of the quadcopter,  $10^{-2} \ \text{rad/s}^2$  for the joint accelerations, has been added.

In all numerical simulation experiments, the control gains are set as $\bm{K}_{p}=\text{diag}([2.2,2.2,2.2])$, $\bm{K}_{v}=\text{diag}([2,2,2])$, $\bm{K}_{\Phi}=\text{diag}([24,24,24])$, $\bm{K}_{\omega}=\text{diag}([16,16,16])$, $\bm{K}_{M,v}=\text{diag}([100,100,100])$, and $\bm{K}_{M,p}=\text{diag}([100,100,100])$.

\begin{figure*}[!t]
		\centering
		\includegraphics[width=\textwidth]{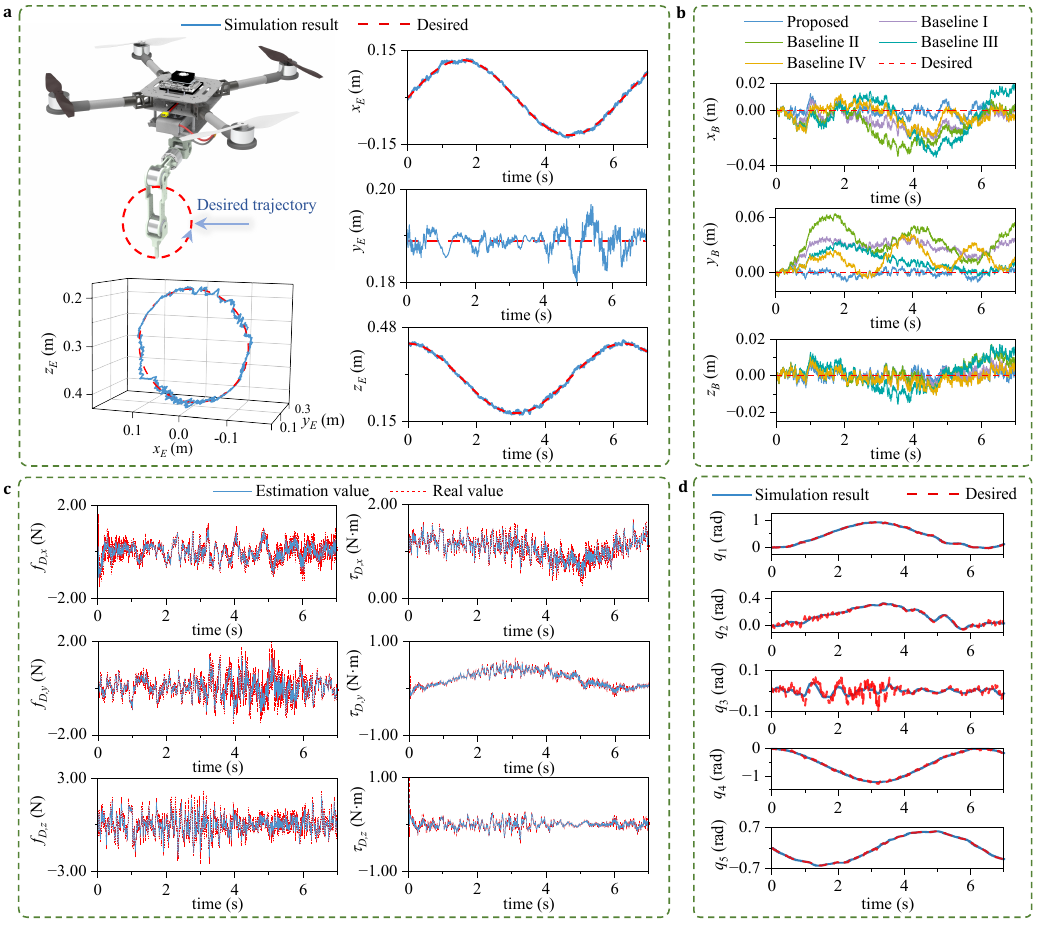}
	\caption{The tracking performance of aerial manipulator under uniformly distributed disturbance. \textbf{a}, Tracking results of the end-effector. \textbf{b}, Tracking results of the quadcopter. \textbf{c}, Estimation results of dynamic coupling force $\bm{f}_D$ and torque $\bm{\tau}_D$. \textbf{d}, Tracking results of robotic arm joints. }
	\label{sim1}
\end{figure*}

\begin{figure*}[!t]	
		\centering
		\includegraphics[width=\textwidth]{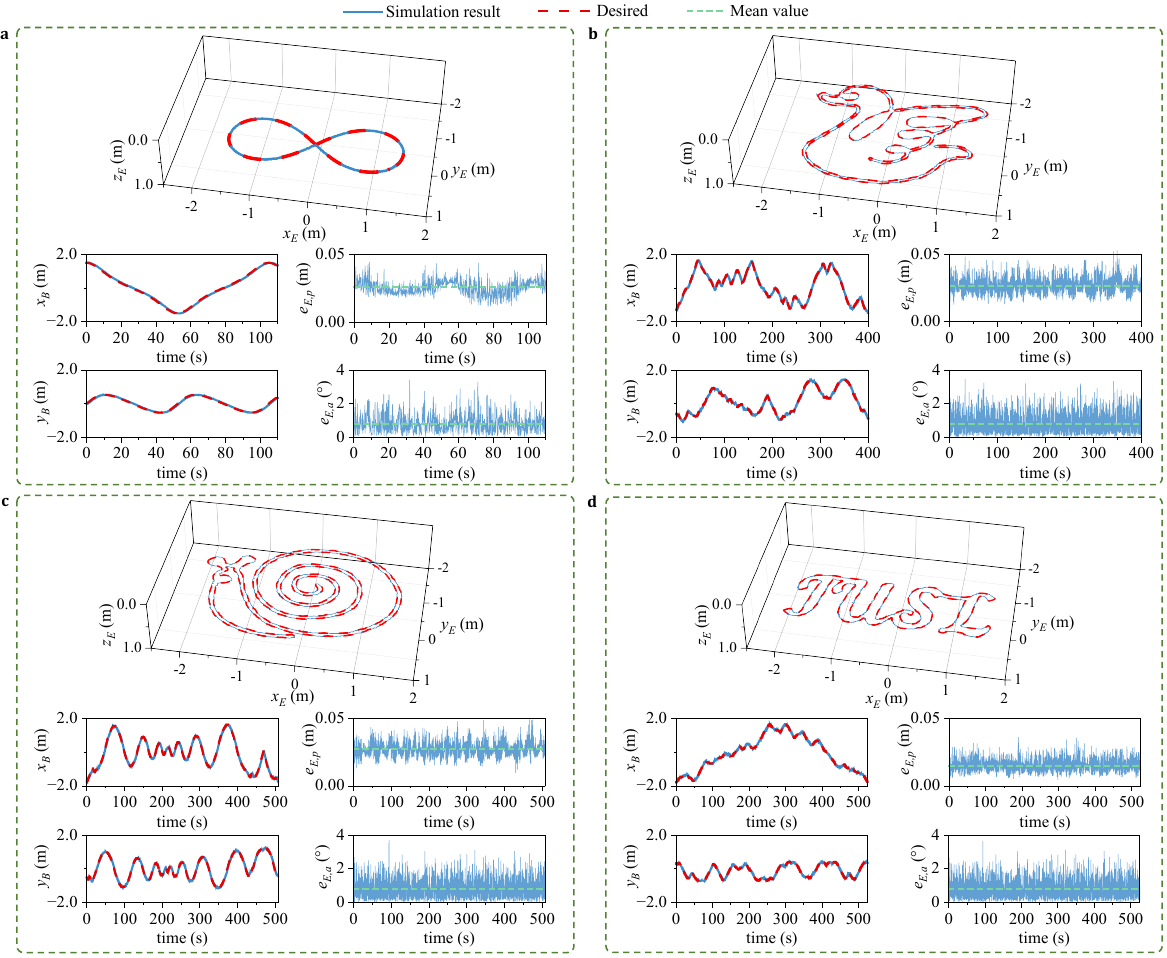}
	\caption{The trajectories of the quadcopter and end-effector. In particular, we let ${e}_{E,p}=\|\bm{e}_{E,p}\|$ denotes the position error of the end-effector, and ${e}_{E,a}=\|\bm{e}_{E,a}\|$ denotes the attitude error of the end-effector. \textbf{a}, Tracking results of the 8-shaped curve (lemniscate of Huygens) \cite{curve8}. \textbf{b}, Tracking results of a simple sketch of a duck. \textbf{c}, Tracking results of a simple sketch of a snail. \textbf{d}, Tracking results of a simple sketch of the fancy letter "IUSL" in English.}
	\label{sim2}
\end{figure*}

\begin{figure*}[!t]	
		\centering
		\includegraphics[width=\textwidth]{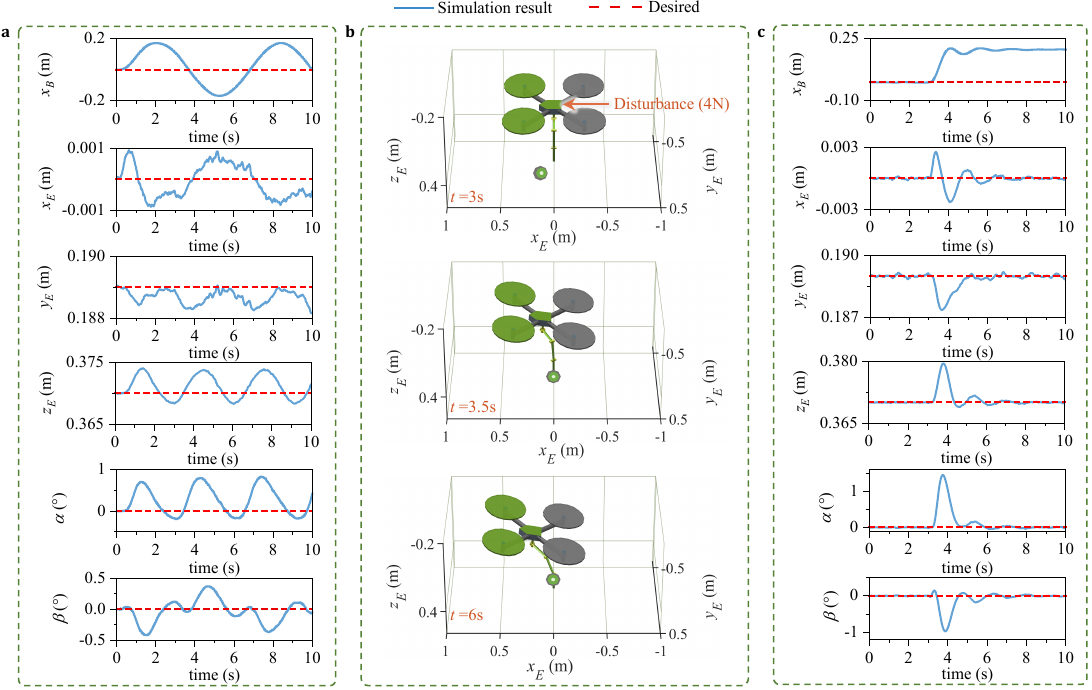}
	\caption{The experimental results of the quadcopter and the end-effector. \textbf{a}, The experimental results under a sinusoidal force disturbance in the direction of $x$-axis with an amplitude of 4 N and an angular frequency of 1 rad/s. \textbf{b}, The visual description of the end-effector pose stabilization under the step force disturbance. \textbf{c}, The experimental results under a step force disturbance in the direction of $x$-axis, applied from 3rd second with a final value of 4 N.}
	\label{sim3}
\end{figure*}

\subsection{Example 1: Disturbance Rejection Experiment}

The control objective of this experiment is for the end-effector to track a circular trajectory on the $x$-$z$ plane with a radius of 12 cm and an angular frequency of 1 rad/s.  During the experiment, the quadcopter maintains a fixed position despite uniformly distributed measurement noise. The range of the position measurement noise is  $[-0.02, 0.02]$ m, and the range of the attitude measurement noise is  $[-5, 5]^\circ$.

This experiment aims to verify the effectiveness of the proposed control framework. The tracking results of the end-effector are shown in Fig~\ref{sim1}a. The mean position error of the end-effector is 0.50 cm, while the maximum position error is 2.1 cm. The mean attitude error is 0.8$^{\circ}$, while the maximum attitude error is 3.3$^{\circ}$. These results indicate that the end-effector can achieve high tracking accuracy even in the presence of position and attitude measurement noise in the quadcopter. In addition, we validate the accuracy of the proposed RNE-based estimation algorithm. As shown in Fig.~\ref{sim1}c, the estimated force and torque can track real values well, with a mean force error of 0.4 N and a mean torque error of 0.1 N$\cdot$m. Moreover, we validate the effectiveness of the robotic arm controller. Fig. \ref{sim1}d shows the tracking performance of the joints.

To further demonstrate the performance of the proposed flight controller, we compare it with four baseline controllers. The first baseline controller is similar to the proposed controller but without compensating for dynamic coupling. The second is based on the inverse dynamic control method proposed in \cite{pierri2018adaptive}. 
The third uses the classic PID control method \cite{2013PD}.  The fourth is designed by the geometric control method \cite{2021switchable}. These control methods are widely used in tracking control for aerial manipulators.
Figure.~\ref{sim1}b shows the experimental results of these five controllers. The mean position error of the proposed method is 0.52 cm. Compared to the other four controllers, the proposed flight controller can reduce the mean position error by 81\%, 86\%, 75\%, and 70 \%, respectively. It illustrates the performance of the proposed flight control method.

\subsection{Example 2: End-Effector Trajectories Tracking}

The objective of this experiment is for the end-effector to accurately track four complex trajectories (see Fig.~\ref{sim2}). Given the extensive range of these trajectories, both the quadcopter and the robotic arm must move together to track the desired position of the end-effector. In particular, in some situations, obtaining velocity commands for the desired trajectory is challenging, which can increase tracking errors. To simulate these conditions, the controller does not use the velocity information of these trajectories.

This experiment aims to validate the effectiveness of coordinated control in cooperation mode.
The experimental results demonstrate the end-effector's ability to track various curves accurately.
As shown in Fig. \ref{sim2}, the mean position errors of the end-effector in these four cases are 2.57 cm, 2.69 cm, 2.71 cm, and 1.51 cm, respectively, and the mean attitude errors are 0.80$^{\circ}$, 0.81$^{\circ}$, 0.82$^{\circ}$, and 0.82$^{\circ}$, respectively. Compared to the first experiment, the tracking accuracy in this experiment is worse due to two main reasons. First, these four curves have large coordinate ranges and complex trajectories, requiring high-velocity changes, which makes them difficult to track. Second, in this experiment, velocity commands of the trajectories are not provided, preventing the system from tracking the desired trajectory promptly.

\subsection{Example 3: End-Effector Pose Stabilization}

The control objective of this experiment is to maintain the end-effector's position and attitude unchanged despite two types of disturbances applied to the quadcopter. The first disturbance is a sinusoidal force in the direction of $x$-axis with an amplitude of 4 N and an angular frequency of 1 rad/s. The second disturbance is a step force along the $x$-axis, beginning at 3~s, and reaching a final value of 4~N. 

Figure \ref{fig_visual_description_avian_head} presents a comparison between end-effector stabilization and avian head stabilization. On the left, a falcon sits on a person's arm, which moves sinusoidally from left to right. The figure on the right displays our simulation results, where a sinusoidal force guides the quadcopter in performing analogous sinusoidal movements to mimic the motion of the human arm. These results show that the proposed algorithm can achieve performance comparable to that of the falcon. Figure~\ref{sim3}a illustrates the tracking results of the quadcopter and the end-effector. Under the sinusoidal force, the maximum displacement of the quadcopter in the direction of $x$-axis is up to 0.17 m, while the maximum position error of the end-effector is 0.004~m. The maximum error in the attitude angles $\alpha$ and $\beta$ are 0.84$^{\circ}$ and 0.42$^{\circ}$, respectively. To further validate the performance, a step force disturbance is also applied to the quadcopter (see Fig.~\ref{sim3}b). The tracking results are shown in Fig.~\ref{sim3}c. The maximum displacement of the quadcopter in the direction of $x$-axis reaches 0.19 m, while the maximum position error of the end-effector is 0.01~m. The maximum errors in the attitude angles $\alpha$ and $\beta$ are 1.44$^{\circ}$ and 0.97$^{\circ}$, respectively.
These results demonstrate that the end-effector can maintain its pose with high precision despite the application of two types of disturbances to the quadcopter.

 We achieve functionalities similar to avian head stabilization by leveraging the manipulation capabilities and high degrees of freedom of the robotic arm. Our proposed control method stabilizes the end-effector's pose for two reasons. First, the RNE-based estimator provides accurate estimates of the dynamic coupling between the quadcopter base and the robotic arm, enhancing the overall precision of the system. Second, the coordinated control algorithm quickly compensates for position and attitude errors in the quadcopter base, enabling end-effector tracking control similar to avian head stabilization.

\section{Conclusion} \label{sec_conclusion}
This paper proposed an avian-inspired approach for end-effector tracking control. It is verified by three numerical experiments. The first experiment demonstrates the proposed arm and overall control scheme achieve accurate position and attitude tracking performance, with a mean end-effector position error of 0.5 cm and a mean attitude error of 0.8$^{\circ}$. Furthermore, the effectiveness of the flight control algorithm is demonstrated, as our proposed algorithm reduces tracking errors by 81\%, 86\%, 75\%, and 70\% compared to the four baseline control methods.
The second experiment verifies that the end-effector can track different complex curves in cooperation mode. However, the trajectory-tracking accuracy in this experiment is lower compared to the first experiment due to the more complex trajectories and the lack of velocity commands.
The third experiment demonstrates that the aerial manipulator can stabilize the end-effector with high precision, similar to how birds stabilize their head position. In the experiment, the mean position error can achieve millimeter-level, and the mean attitude error is within 1$^{\circ}$. This demonstrates the end-effector's ability to resist quadcopter error disturbances and maintain high-level position and attitude accuracy, achieving avian head-like pose stabilization.

Future research will be directed towards two key areas. First, while we have successfully implemented high-precision control of the end-effector, our current control law lacks the integration of a specific dynamic model for avian head stabilization. Incorporating these dynamics into aerial manipulators presents a significant avenue for advancement. 
Second, our proposed method does not consider aggressive manipulation behaviors typically observed in birds, such as Eagle’s catch and Monkey-Bar behaviors \cite{MartiSaumell2021FullBodyTN}. Subsequent studies will aim to investigate this aspect through real-world experiments.

\bibliography{REF}

\end{document}